\documentclass{article}

\usepackage[preprint]{neurips_2025}

\usepackage[utf8]{inputenc} 
\usepackage[T1]{fontenc}    
\usepackage{hyperref}       
\usepackage{url}            
\usepackage{booktabs}       
\usepackage{amsfonts}       
\usepackage{nicefrac}       
\usepackage{microtype}      
\usepackage{xcolor}         

\usepackage{amsmath}
\usepackage{amssymb}
\usepackage{mathtools}
\usepackage{amsthm}
\usepackage{enumitem}
\usepackage{caption}
\usepackage{mathrsfs}
\usepackage{bm}
\theoremstyle{plain}
\newtheorem{theorem}{Theorem}

\newtheorem{lemma}{Lemma}

\theoremstyle{definition}

\theoremstyle{remark}
\newtheorem{remark}[theorem]{Remark}

\newcommand{\R}{\mathbb{R}}

\newcommand{\Id}{\mathbf{I}}
\newcommand{\bA}{\mathbf{A}}
\newcommand{\bB}{\mathbf{b}}
\newcommand{\bx}{\mathbf{x}}
\newcommand{\by}{\mathbf{y}}
\newcommand{\br}{\mathbf{r}}
\newcommand{\bk}{\mathbf{k}}
\newcommand{\grad}{\nabla}
\newcommand{\diver}{\nabla\!\cdot}
\newcommand{\curl}{\nabla\times}

\newcommand{\conv}{\ast}

\title{Neural Contrast Expansion for Explainable Structure-Property Prediction and Random Microstructure Design}
\workshoptitle{Accelerated Materials Design}

\author{
  Guangyu Nie\textsuperscript{1}\quad
  Yang Jiao\textsuperscript{2}\quad
  Yi Ren\textsuperscript{1}\\
  \textsuperscript{1}Mechanical Engineering, Arizona State University\\
  \textsuperscript{2}Materials Science, Arizona State University\\
  Tempe, AZ 85287, USA\\
  \texttt{\{gnie1,yjiao13,yiren\}@asu.edu}
}

\begin{document}

\maketitle

\begin{abstract}
Effective properties of composite materials are defined as the ensemble average of property-specific PDE solutions over the underlying microstructure distributions. Traditionally, predicting such properties can be done by solving PDEs derived from microstructure samples or building data-driven models that directly map microstructure samples to properties. The former has a higher running cost, but provides explainable sensitivity information that may guide material design; the latter could be more cost-effective if the data overhead is amortized, but its learned sensitivities are often less explainable. With a focus on properties governed by linear self-adjoint PDEs (e.g., Laplace, Helmholtz, and Maxwell curl-curl) defined on bi-phase microstructures, we propose a structure-property model that is both cost-effective and explainable. Our method is built on top of the strong contrast expansion (SCE) formalism, which analytically maps $N$-point correlations of an unbounded random field to its effective properties. Since real-world material samples have finite sizes and analytical PDE kernels are not always available, we propose Neural Contrast Expansion (NCE), an SCE-inspired architecture to learn surrogate PDE kernels from structure-property data. For static conduction and electromagnetic wave propagation cases, we show that NCE models reveal accurate and insightful sensitivity information useful for material design. Compared with other PDE kernel learning methods, our method does not require measurements about the PDE solution fields, but rather only requires macroscopic property measurements that are more accessible in material development contexts. 
\end{abstract}

\section{Introduction}
\vspace{-0.1in}
This paper is concerned with \textit{effective properties} of composite materials, which are macroscopic measurements affected by the materials' random microstructure.
We seek a model trained on structure–property data that \textit{explains} what and how microstructural features affect the effective property, providing insights for microstructural design.
Some structure–property relations are intuitive (e.g., the effective conductivity of a composite is dominated by the connectivity of its conductive filler); others have been derived from first principles by materials scientists, e.g., minimizing energy dissipation of light propagation within an optical sensor during wave propagation can be achieved by a hyperuniform arrangement of the random microstructure~\citep{Leseurhyperuniform2016}.
However, for novel materials and target properties, such human-comprehensible knowledge may not be accessible from first principles due to the multiscale and nonlinear nature of the governing equations.

Conventionally, property prediction is achieved by either solving the governing PDEs or building data-driven surrogates: The former requires knowledge about the PDEs and often a significant amount of compute. Although when differentiable, the solution process yields explainable sensitivities; the latter is more cost-effective with amortized data overhead, but generally lacks explainable sensitivities because it is detached from the physics. Structure–property mappings that leverage advantages of both have been developed in the context of effective-medium theorems; among these, we focus on the strong contrast expansion (SCE), which is accurate and explicit for linear PDEs. SCE states that \textit{linear} effective properties of \textit{unbounded} random fields are analytical functions given by spatial convolutions between the PDE kernel (i.e., Green's functions or their Hessians) and the infinite series of $N$-point correlation functions (NPCFs), which characterize random fields at increasing orders and in semantic ways (Fig.~\ref{figure: microstructure_vs_2pcf} and ~\citep{Cheng2024spectra}). For example, connectivity of a contrast phase within a composite material is captured by its 2-point correlation function (Fig.~\ref{figure: microstructure_vs_2pcf}). However, SCE has practical limitations: real samples are finite (introducing boundary effects not captured by SCE), and the relevant kernels may lack closed forms or may not exist.

\vspace{-0.1in}
\paragraph{Method.} To address these limitations, we investigate Neural Contrast Expansion (NCE), a learnable model with an SCE-guided architecture. We encode physics by mapping structure to property as a convolution of NPCFs with a learnable kernel, and estimate that kernel from structure–property data. We use two-dimensional static conduction and electromagnetic wave propagation in bi-phase composites as case studies, whose governing equations (Laplace and Maxwell curl-curl, respectively) have analytic kernels under mild conditions. Numerical results show that NCE generalizes comparably to convolutional neural networks and neural operators, while producing sensitivities that are substantially more explainable and accurate than purely data-driven counterparts. Compared with supervised methods that learn Green's functions or closures~\citep{Boull2022datagreen, PRAVEEN2023green,gin2020deepgreen}, NCE does not require measurements of PDE solution fields; it only needs macroscopic property measurements that are more accessible in practice. Compared with physics-informed methods for learning Green's functions~\citep{hao2024green}, NCE avoids numerical issues introduced by the singularity of Green's functions through a cavity technique inherited from SCE.

\begin{figure}[t]
  \centering
  \includegraphics[width=\columnwidth]{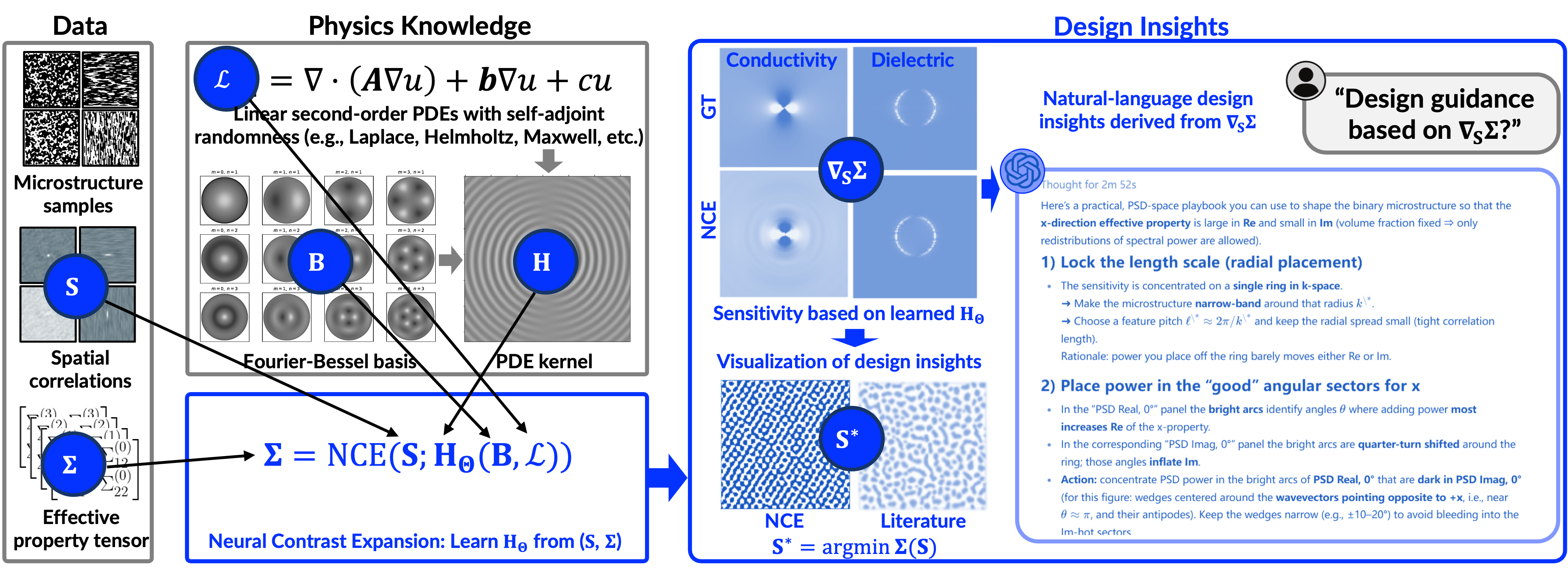}
  \caption{This paper improves the workflow for deriving composite material design insights from measurement data regarding microstructures (``$S$'') and their effective properties (``$\boldsymbol{\Sigma}$''). We propose Neural Contrast Expansion (NCE), which learns PDE kernels (``$\mathbf{H}$'') based on data and physics knowledge (PDE ``$\mathscr{L}$'' and basis ``$\mathbf{B}$'') to provide accurate sensitivity information (``$\nabla_S \boldsymbol{\Sigma}$''). The sensitivity derived from NCE is much more accurate than that of purely data-driven models (see Fig.~\ref{figure: conduction_kernel}), and in turn, provides explainable design insights through LLM and optimal microstructure design (``$S^*$''). For optical material design with minimal energy loss, design insights derived from NCE lead to hyperuniform microstructures, consistent with recent literature~\cite{To18a}. 
  }
   \label{figure: 1}
   \vspace{-0.1in}
\end{figure}

\vspace{-0.1in}
\paragraph{Scope and limitations.} In this work we restrict attention to randomness entering the zeroth- or second-order terms of linear second-order PDEs. These two classes admit a unified SCE because their perturbations appear as self-adjoint modifications to a coercive baseline operator and generate even convolutional kernel structures that close under ensemble averaging.
In contrast, first-order randomness (random drift/advection) produces skew-adjoint, odd-kernel perturbations that, in general, break this closure: the resulting effective operator depends on nonlocal flow statistics rather than solely on local correlations, and the convergence of SCE is generally not guaranteed (see App.~\ref{sec: first_order_rand} for explanation).
We note that there are niche yet valuable applications of NCE for the focused randomness classes beyond materials science: For climate science, elliptic/parabolic PDEs for heat, moisture, or tracer transport (e.g., diffusion or advection–diffusion) include spatially heterogeneous parameter fields (e.g., diffusivities, reaction rates, subgrid closures). NCE could be used to learn kernels that predict how fast heat, moisture, or pollution spreads and mixes when tiny swirls of air or water cannot be resolved.
For neuroscience, the quasi-static electromagnetic forward problem is modeled by Laplace/Poisson equations with spatially varying and anisotropic tissue conductivities. When solution fields (e.g., MRI) are unavailable, NCE could be used to estimate how electrical signals travel through different parts of the brain using sensor-level measurements.
For cosmology, the gravitational potential obeys the Poisson equation and the matter density is a random field with rich $N$-point correlations. NCE can learn kernels that map these correlations to large-scale effective observables (e.g., summary statistics related to lensing or clustering) without requiring full field solutions. The advantage of NCE against other kernel-learning models with stronger cosmology priors (e.g., effective field theorem~\citep{cabass2022snowwhite} and Halo models~\citep{Asgari2023halo}) is yet to be understood.  

\vspace{-0.1in}
\section{Related Work}\vspace{-0.1in}
\paragraph{Learning Green's functions from field-level data.}
A substantial line of work learns Green's function of linear PDEs directly from source–solution data, using either physics-regularized setups or supervised excitation–response pairs. Representative approaches include multiscale neural networks tailored to the singular/multiscale structure of the Green's function~\citep{hao2024green}, rational-network regression of Green's functions~\citep{Boull2022datagreen}, and low-rank/operator–SVD interpolation of learned kernels~\citep{PRAVEEN2023green}; extensions like DeepGreen address nonlinear BVPs via latent linearization~\citep{gin2020deepgreen}. These methods assume access to dense fields or high-fidelity simulations for training.
\vspace{-0.1in}
\paragraph{Nonlocal closures and operator learning.}
In parallel, closure and operator-learning work fits nonlocal convolution kernels (e.g., eddy-diffusivity ~\citep{PRF2023nonlocal,liu2023eddy}, nonlocal constitutive laws~\citep{Sanderse2024closure}) so a coarse model matches resolved simulations, enforcing physics constraints such as symmetry/reciprocity, decay, and positivity. Examples include systematic constructions of nonlocal eddy-diffusivity operators in fluids and oceanography~\citep{PRF2023nonlocal,liu2023eddy}, and data-driven learning of nonlocal constitutive/transport kernels from high-fidelity simulations~\citep{you2020datadrivennonlocal,cabass2022snowwhite}.
\vspace{-0.1in}
\paragraph{Our setting and gap.}
By contrast, many experimental regimes (notably in materials) lack field-level supervision but do provide structure–property datasets and statistics of heterogeneous coefficients. Our NCE targets this regime: We (i) learn the PDE kernel at a chosen coarse resolution from ensemble/effective measurements (rather than field snapshots), and (ii) plug that kernel into SCE to obtain closed-form maps and sensitivities from coefficient correlations to effective properties. 

\begin{figure}[t]
  \centering
  \includegraphics[width=\columnwidth]{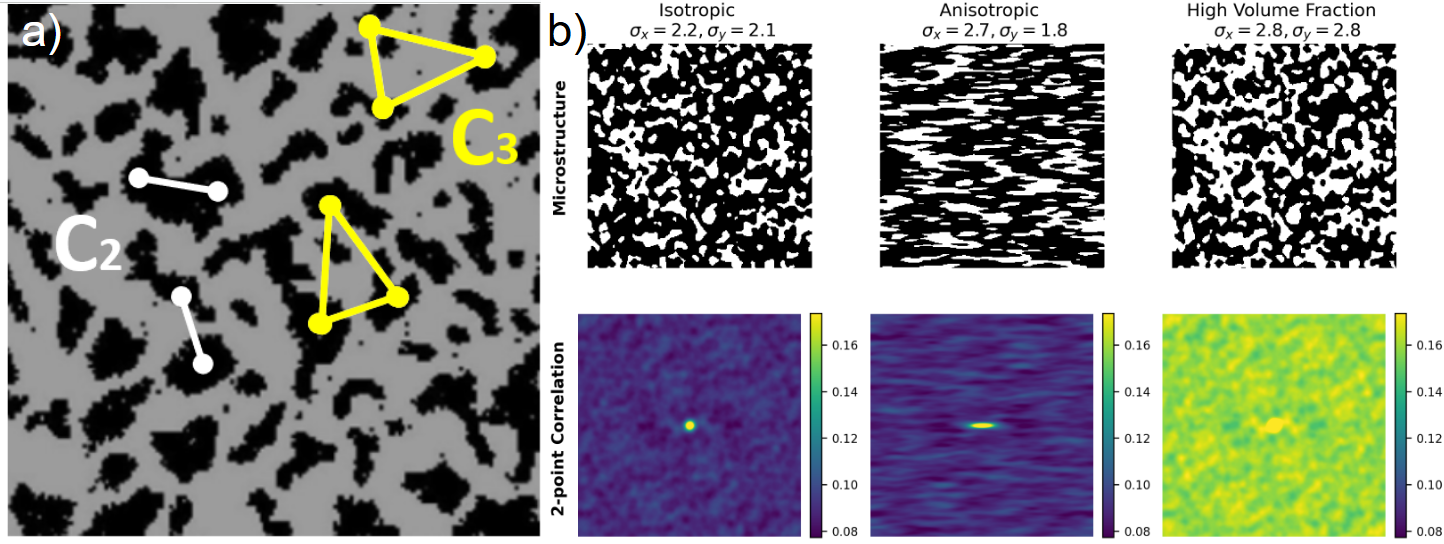}
  \caption{(a) Illustration of two- and three-point configurations ($\mathcal{C}_2$ and $\mathcal{C}_3$) in a binary microstructure. The black and gray regions represent two distinct phases of the medium, while the white line segments and yellow triangles mark representative 2-point and 3-point configurations. (b) Binary microstructures with distinct features: (left) isotropic, (middle) anisotropic with horizontal elongation, and (right) high-volume fraction. The parameters $\sigma_x$ and $\sigma_y$ indicate the effective directional conductivity along the x and y axes. The second row shows the corresponding 2PCFs.}
   \label{figure: microstructure_vs_2pcf}
   \vspace{-0.1in}
\end{figure}

\vspace{-0.1in}
\section{Preliminaries} 
\label{sec:prelim}
\vspace{-0.1in}
We shall start with the formal definitions of random fields, effective properties, NPCFs, and SCE. 
\vspace{-0.1in}
\paragraph{Notations.} Let the relative position between two vertices $(\mathbf{i},\mathbf{j})$ be $\mathbf{r}(\mathbf{i},\mathbf{j}) = \mathbf{i}-\mathbf{j}$, $r = \|\mathbf{r}\|_2$, and $\mathbf{n} = \mathbf{r}/r$. $\delta(\mathbf{r})$ is the Dirac delta function. $\mathbb{E}[\cdot]$ denotes the ensemble average and $\langle \cdot\rangle$ the spatial average. $\nabla$, $\nabla \cdot$, and $\nabla \times$ are gradient, divergence, and curl operators, respectively. $\mathcal{A}/a$ is set $\mathcal{A}$ subtracts element $a$.
\vspace{-0.1in}
\paragraph{Random field.} We consider stationary and ergodic random fields defined on a grid that describe binary-phase microstructures: Let a random field be $X = \{X_\mathbf{i} \in \{0,1\} : \mathbf{i} \in \mathcal{G}\}\), where $X_\mathbf{i}$ is a binary random variable located at vertex $\mathbf{i}$ on a $d$-dimensional grid $\mathcal{G} \subset \mathbb{Z}^{d}$ where $G:=|\mathcal{G}|$. The grid $\mathcal{G}$ is taken as a uniform discretization of a unit d-dimensional hypercube $[0, 1]^d$. 
Here $X_\mathbf{i}=1$ indicates the inclusion (contrast) phase and $X_\mathbf{i}=0$ the matrix (reference) phase.
$X$ is \textit{stationary} if all of its moments are invariant to linear translations, and is \textit{ergodic} if its spatial and ensemble averages are equal. 
Ergodicity allows us to focus the analysis on a single realization $x \sim X$ defined on a large enough grid.
The volume fraction of the contrast phase is $\phi=\mathbb{E}[X_\mathbf{i}]=\langle X_\mathbf{i}\rangle$.
For a binary medium $\phi_1=\phi$, $\phi_0=1-\phi$ denote the phase 1 and phase 0 volumetric proportions, respectively.
\vspace{-0.1in}
\paragraph{Effective property.} A linear effective property of the composite is defined via a homogenization of the governing PDE solution. Here we use static conduction and wave propagation as two representative examples, where randomness of microstructures affect the second- and zeroth-order terms of the respective PDEs.  
\textbf{Static conduction:} In a static conduction problem with a reference conductivity $\sigma_0$ and inclusion conductivity $\sigma_1$, if an electric field $\mathbf{E}_0$ is applied at the boundary $\partial \mathcal{G}$, the local electric potential $\Phi(\mathbf{i},x)$ in the sample $x$ satisfies the elliptic equation 
\begin{equation*}
    \nabla\cdot[\sigma(x_\mathbf{i})\nabla \Phi]=0\quad\text{in }\mathcal{G},
\end{equation*}
with $-\nabla \Phi|_{\partial \mathcal{G}}=\mathbf{E}_0$. 
The induced current density is $\mathbf{J}(\mathbf{i},x)=\sigma(x_\mathbf{i})\mathbf{E}(\mathbf{i},x)$ for $\mathbf{i} \in \partial \mathcal{G}$, whose spatial average $\langle \mathbf{J} \rangle(x)$ equals the ensemble-averaged flux due to ergodicity. 
The effective conductivity tensor $\boldsymbol{\Sigma_e}$ is then defined by Ohm’s law $\langle \mathbf{J}\rangle(x)=\boldsymbol{\Sigma}_e(x)\mathbf{E}_0$. 
In practice, one can compute $\boldsymbol{\Sigma}_e(x)$ by solving the PDE (or measuring $\langle \mathbf{J}\rangle(x)$ experimentally) under $d$ linearly independent boundary conditions and averaging fluxes. 
\textbf{Wave propagation:} We now consider the time-harmonic Maxwell equations. 
In a nonmagnetic composite ($\mu=\mu_0$) with spatially varying permittivity $\mathbf{\epsilon}(x_{\mathbf{i}})$ and $e^{-i\omega t}$ convention, the electric field satisfies the curl–curl equation
\[
\nabla\times\nabla\times \mathbf{E}(\mathbf{i},x)\;-\;\omega^2\mu_0\,\mathbf{\epsilon}(x_{\mathbf{i}}; \mathbf{k})\,\mathbf{E}(\mathbf{i},x)=\mathbf{0}
\quad\text{in }\mathcal{G},
\]
with appropriate radiation/periodic boundary conditions and an incident plane wave of wave vector $\mathbf{k}$. 
The effective dynamic dielectric tensor $\boldsymbol{\epsilon}_e(x, \mathbf{k})$ is defined by the macroscopic constitutive relation $\langle \mathbf{D}\rangle (x) \;=\; \boldsymbol{\epsilon}_e(x,\mathbf{k})\,\langle \mathbf{E}\rangle(x)$
where $\mathbf{k}$ is the wave vector of the incident radiation.
Unless stated otherwise, we assume a homogeneous reference medium on $\mathbb{R}^d$ with translation-invariant boundary conditions; consequently the Green's function is homogeneous.
We note that in these examples, randomness appears in the second-order term for conduction and in the zeroth-order term for wave propagation. As a result, the PDE kernel for the former will turn out to be the Hessian of Green's function, while that of the latter be the Green's function itself.


\vspace{-0.1in}
\paragraph{$N$-point correlation function.} The NPCF of a random field describes the probability of simultaneously finding a specific $N$-tuple of points in the inclusion phase.
Formally, for any configuration $c=\{\mathbf{i}_1,\dots,\mathbf{i}_N\}$ of $N$ points (with $\mathbf{i}_1=0$ as a reference origin), the $N$-point function is defined as 

\begin{equation}
S_N(X;c) \;=\; \mathbb{E}\big[X_{\mathbf{i}_1}X_{\mathbf{i}_2}\cdots X_{\mathbf{i}_N}\big] \;=\; \frac{1}{G}\sum_{\{\mathbf{j}_2,\ldots,\mathbf{j}_N\}} x_{\mathbf{j}_1}x_{\mathbf{j}_2}\cdots x_{\mathbf{j}_N}, \tag{2} \label{eq:NPCF} 
\end{equation}

where the second equality is enabled by ergodicity of $X$ and the sum runs over all placements $\{\mathbf{j}_1,\ldots,\mathbf{j}_N\}$ congruent to $c$ in the domain. Dependency on the sample $x$ will be omitted when possible.
For $N=2$, $S_2(\mathbf{r})$ (with $\mathbf{r} = \mathbf{i}_2-\mathbf{i}_1$) is the two-point correlation giving the probability that two points separated by vector $\mathbf{r}$ are both in the inclusion phase. 
The lower-order correlations are contained in higher-order ones, i.e., $\mathcal{C}_N \subset \mathcal{C}_{N+1}$ where $\mathcal{C}_N$ is the set of all $N$-point configurations. 
In particular, knowledge of all $N$-point functions up to $N=\infty$ fully characterizes the microstructure.
In practice, low-order correlations (e.g. $N=2,3$) already capture important morphological features such as phase connectivity and clustering. 

\vspace{-0.1in}
\paragraph{Strong Contrast Expansion.} SCE treats the contrast medium as a perturbation about the reference medium and systematically accounts for multiple scattering or interactions via successive convolution integrals, resulting in an analytical expression of effective properties as series expansions in terms of NPCFs convolved with a PDE kernel~\citep{torquato2002random}.
Here we briefly outline SCE in terms of static conduction and electromagnetic wave propagation. See detailed derivation in App.~\ref{sec:sce}.
\textbf{Static conduction:} Briefly, SCE introduces a polarization field based on the spatial perturbation of the contrast phase: 
\[
\mathbf{P}(\mathbf{i},x) = \bigl(\sigma(x_\mathbf{i})-\sigma_0\bigr)\mathbf{E}(\mathbf{i},x),
\]
and derives a fixed-point iteration of the polarization field from the Laplace equation: 
\begin{equation}
\mathbf{P}(\mathbf{i},x) = a(x_\mathbf{i})\,\mathbf{E}_0 + a(x_\mathbf{i})\,\Delta V \sum_{\mathbf{j}\in\mathcal{G}/\mathbf{i}} \mathbf{H}(\mathbf{i}-\mathbf{j})\,\mathbf{P}(\mathbf{j},x),
\label{eq_P}
\end{equation}
where $\Delta V$ is the $d$-dimensional unit volume occupied by each vertex of $\mathcal{G}$,
\begin{equation}
a(x_\mathbf{i}) = d\sigma_0 \beta_\sigma\,x_\mathbf{i}, \quad \beta_\sigma = \frac{\sigma_1-\sigma_0}{\sigma_1+(d-1)\sigma_0}\in\Bigl[-\frac{1}{d-1},1\Bigr),
\label{eq_beta_conduction}
\end{equation}
and
\[
\mathbf{H}(\mathbf{r}) = \frac{d\,\mathbf{n}\mathbf{n}^\top-\mathbf{I}}{\Omega_d\,\sigma_0\,r^d}
\]
is the kernel (Hessian of Green's) for 2D conductivity with $\Omega_d$ being the $d$-dimensional total solid angle. Iterative substitution of Eq.~\eqref{eq_P} into itself and a spatial averaging lead to the following series expansion of the effective conductivity:
\begin{equation}
\beta_\sigma^2\phi^2(\boldsymbol{\Sigma}_e-\sigma_0\mathbf{I})^{-1} \Bigl(\boldsymbol{\Sigma}_e+(d-1)\sigma_0\mathbf{I}\Bigr) = \beta_\sigma\phi\,\mathbf{I} - \sum_{n=2}^{G}\mathbf{A}_n\,\beta^n_\sigma,
\label{eq_SCE}
\end{equation}
where each $\mathbf{A}_n$ is an $n$-th order convolution integral of total correlation functions (constructed from $S_n$) with $(n-1)$ factors of the operator kernel $\mathbf{T} = \Omega_d \sigma_0 \mathbf{H}$
\footnote{For conciseness, we use ``1'' instead $\mathbf{i}_1$ when appropriate.}:
\begin{equation}\label{eq:A}
    \mathbf{A}_n =\left(-\frac{1}{\rho}\right)^{n-2} \left(\frac{d}{\Omega_d} \right)^{n-1} \int d2\cdots dn \Delta(1,\cdots,n) \mathbf{T}(1,2)\mathbf{T}(2,3)\cdots \mathbf{T}(n-1,n),
\end{equation}
where $\Delta(1,\cdots,n)$ is the $n$-order total correlation defined by the configuration $\{1,\cdots,n\}$:
\begin{equation*}
    \Delta(1,\cdots,n) = \begin{vmatrix}
S_2(x, \{1, 2\}) 
    & S_1(x, 2) 
    & \cdots 
    & 0 \\
S_3(x, \{1, 2, 3\}) 
    & S_2(x, \{2, 3\}) 
    & \cdots 
    & 0 \\
\vdots 
    & \vdots 
    & \ddots 
    & \vdots \\
S_n(x, \{1,\ldots,n\}) 
    & S_{n-1}(x, \{1,\ldots,n\}) 
    & \cdots 
    & S_2(x,\{n-1,n\})
\end{vmatrix}
\end{equation*}
At second order, the leading correction involves the two-point function $S_2$ via an integral of $S_2(\mathbf{r})$ against $\mathbf{H}(\mathbf{r})$, linking phase connectivity to effective conductivity.
\textbf{Wave propagation:} For a macroscopically anisotropic bi-phase medium with matrix permittivity $\mathbf{\epsilon}_0$ and inclusion permittivity $\mathbf{\epsilon}_1$, one obtains an SCE formulation for the dielectric constant $\boldsymbol{\epsilon}_e$ dependent on the wave number $\mathbf{k}$. This formulation turns out to be identical to that of the static conduction problem~\citep{Torquato2021a}:
\begin{equation}
\label{eq:SCE_wave_series}
\beta^2_\epsilon\phi^2(\boldsymbol{\epsilon}_e(\mathbf{k}) -\epsilon_0\mathbf{I})^{-1} \Bigl(\boldsymbol{\epsilon}_e(\mathbf{k})+(d-1)\sigma_0\mathbf{I}\Bigr) = \beta_\epsilon\phi\,\mathbf{I} - \sum_{n=2}^{G}\mathbf{A}(\mathbf{k})_n\,\beta^n_\epsilon,
\end{equation}
where
\begin{equation}
\label{eq:beta_def}
\beta_\mathbf{\epsilon} \;=\; \frac{\mathbf{\epsilon}_1 - \mathbf{\epsilon}_0}{\mathbf{\epsilon}_1 + (d-1)\mathbf{\epsilon}_0}.
\end{equation}
Different from the conduction case, since microstructural randomness in permittivity only affects the zeroth-order term of the Maxwell equation,  the kernel becomes the Green's function itself. In 2D and when local material properties ($\mu$, $\epsilon_0$, and $\epsilon_1$) are isotropic, this Green's function $H^{(0)}(\mathbf{r})$ has been derived as:
\begin{equation}
\label{eq:H2D}
H^{(0)}_{ij}(\mathbf r)
\;=\;
\frac{i}{4\,\mathbf{\epsilon}_0}\Bigg[
\Big(k_0^2 \mathcal{H}^{(1)}_0(k_0 r) - \frac{k_0}{r} \mathcal{H}^{(1)}_1(k_0 r)\Big)\delta_{ij}
\;+\;
k_0^2 \mathcal{H}^{(1)}_2(k_0 r)\,\mathbf{n}_i  \mathbf{n}_j
\Bigg]
\end{equation}
where $\mathbf{k}_0=\omega\sqrt{\mu_0\varepsilon_0}$ and $\omega$ are the wave vector and frequency of the incident radiation respectively, $k_0=\|\mathbf k_0\|$, and $\mathcal{H}^{(1)}_\nu$ is the Hankel function of the first kind of order $\nu$. It should be noted that when local permittivity constants are anisotropic, this analytical Green's does not exist ~\citep{anisotropic2020Mikki}. We also note that in both cases, the convolution integral $\mathbf{A}_n$ is defined on $\mathcal{G}/\mathcal{N}(0)$, where $\mathcal{N}(0)$ is a cavity around around $r = 0$, and is therefore proper. This is because with proper choice of $\mathcal{N}(0)$, e.g., sphere or cylinder, the improper convolution integral over $\mathcal{N}(0)$ can have an analytical form and moved out of the integral. In plain language, for the purpose of predicting effective properties, we do not need to compute an accurate convolution integral within $\mathcal{N}(0)$. This sets our method apart from existing works on learning Green's functions for predicting PDE solution fields~\citep{hao2024green}, where singularity of the Green's function becomes a learning challenge. 

\vspace{-0.1in}
\section{Neural Contrast Expansion}
\label{sec:nce}
\vspace{-0.1in}
Building upon the SCE framework, we propose Neural Contrast Expansion (NCE), a data-driven learnable architecture for approximating PDE kernels through structure-property data, circumventing the theoretical assumptions of SCE and allowing explainable sensitivity analysis for structure-property mappings where the kernel cannot be analytically derived (e.g., in the case of anisotropic permittivity in Maxwell's) or does not exist (e.g., when a PDE is mildly nonlinear). Our key insights are: (1) When microstructural randomness only affects one of the second- and zeroth-order terms of the governing PDE, SCE provides a universal functional relationship between the NPCFs and the effective property, evidenced by Eq.~\eqref{eq_SCE} and Eq.~\eqref{eq:SCE_wave_series}, and (2) the derivation of this functional relationship does not require an explicit kernel. In the following we explain the choice of a hypothesis space for learning kernels, and regularization considerations for promoting explainability and physics consistency.  

\vspace{-0.1in}
\paragraph{Bessel–Fourier Kernel Parameterization.} NCE approximates \(\mathbf{H}(\mathbf{r})\) via a surrogate \(\hat{\mathbf{H}}(\mathbf{r})\) with radial-angular decomposition. For 2D problems, we have
\begin{equation}
    \hat{\mathbf{H}}(\mathbf{r}) =  \begin{pmatrix}
        \hat{\mathbf{H}}_{11}(r, \theta) & \hat{\mathbf{H}}_{12}(r, \theta) \\
        \hat{\mathbf{H}}_{21}(r, \theta) & \hat{\mathbf{H}}_{22}(r, \theta)
    \end{pmatrix},
\end{equation}
where $\mathbf{r}=(x,y)$ and $\theta = \arctan_2(y, x)$. We consider a Bessel-Fourier space for each entry:
\begin{equation}
\hat{\mathbf{H}}_{ij}(r, \theta)
\;=\;
r^{-\alpha_{\text{env}}}
\sum_{n=0}^{N}\sum_{m=1}^{M}
\Big(
C^{(R)}_{ij,n,m}
\;+\; i\,C^{(I)}_{ij,n,m}
\Big)\,
J_n(\alpha_{n,m}\,r)\,
\psi_n(\theta),
\label{eq:BesselFourierExpansion}
\end{equation}
where $J_n$ is the Bessel function of order $n$, and $\psi_n(\theta)$ is an angular Fourier basis mode $e^{i n\theta}$. This basis is chosen because solutions to isotropic PDEs in 2D naturally separate into radial Bessel and angular Fourier modes~\citep{MathematicalMethodsforPhysicists}.
The coefficients $C^{(R)}_{ij,n,m}$ and $C^{(I)}_{ij,n,m}$ are real learnable parameters representing the cosine- and sine-phase weights of each basis mode. 
The coefficient $C^{(R)}+iC^{(I)}$ thus encode complex amplitudes for modes $J_n(\alpha_{n,m} r)e^{i n\theta}$.
A multiplicative envelope $r^{-\alpha_{\text{env}}}$ imposes the appropriate physical decay, with a learnable ${\alpha_{\text{env}}}$.
The radial wavenumbers $\alpha_{n,m}$ are treated as additional learnable parameters, initialized to span a range of physically relevant length scales.

\vspace{-0.1in}
\paragraph{Regularization for explainability.} An $\ell_1$-regularization penalty is applied to all coefficients $C^{(R)}$ and $C^{(I)}$, encouraging sparsity in the active Bessel–Fourier modes. This improves interpretability by forcing only the most salient radial and angular components to represent the kernel, while reducing overfitting.
\vspace{-0.1in}
\paragraph{Regularization for physics consistency.} We impose a physics-informed regularization to ensure that $\hat{\mathbf{H}}(\mathbf{r})$ is consistent with the target PDE. 
Recall that the Green's function $\mathbf{G(\mathbf{r})}$ of a linear differential operator $\mathscr{L}$ satisfies 
\begin{equation}\label{eq:physics_regularization}
    \mathscr{L}[\mathbf{G(\mathbf{r})}] = \delta(\mathbf{r}),
\end{equation}
and the relationship between $\hat{\mathbf{H}}(\mathbf{r})$ and $\mathbf{G(\mathbf{r})}$ is known: When randomness affects the second-order PDE term, $\mathbf{H}(\mathbf{r})$ is the Hessian of $\mathbf{G(\mathbf{r})}$, and when randomness affects the zeroth-order term, $\mathbf{H}(\mathbf{r}) = \mathbf{G(\mathbf{r})}$.
Therefore the residual of Eq.~\eqref{eq:physics_regularization} is considered a regularization term in the learning of $\hat{\mathbf{H}}(\mathbf{r})$ to embed the PDE's defining property directly into the learning, ensuring that $\hat{\mathbf{H}}(\mathbf{r})$ remains within the physically admissible family determined by the target PDE. 

To summarize, with dataset \(\mathcal{D} = \{(S_n^{(m)}, \boldsymbol{\Sigma}_e^{(m)})\}_{m=1}^M\), the loss is defined as:
\begin{equation}
\begin{aligned}
\label{eq:loss function}
L(\Theta) &= \frac{1}{M} \sum_{m=1}^M  \left\| D(S_n^{(m)}) - \hat{D}(S_n^{(m)}; \Theta) \right\|_2^2 
 +{ \lambda_1 \|C\|_{1} \;+\; \lambda_2 \|\mathscr{L}[\hat{\mathbf{G}}(\mathbf{r})] - \delta(\mathbf{r})\|^2} \\&+ \mathbf{1}_{\text{Hessian}}\;\lambda_3\;\frac{1}{|\mathcal{G}|}\sum_{(x,y)\in\mathcal{G}}\sum_{i,j,k}\left|\partial_x \hat{\mathbf{H}}_{ij}(x,y) - \partial_y \hat{\mathbf{H}}_{ik}(x,y)\right|^2,
\end{aligned}
\end{equation}

where $D$ and $\hat{D}$ are defined as:
\begin{align*}
D(S_n) &= \beta^2 \phi^2 \left({\boldsymbol{\Sigma}}_e(S_n) - \sigma_0 \mathbf{I}\right)^{-1} \left({\boldsymbol{\Sigma}}_e(S_n) + \sigma_0 \mathbf{I}\right), \quad
\hat{D}(S_n; \Theta) &= \beta \phi \mathbf{I} - \sum_{n=2}^N \mathbf{A}_n(S_n; {\Theta}) \beta^n.
\end{align*}
We choose to fit $D$ rather than  regressing $\mathbf{\Sigma}_e$ to avoid numerical instabilities, since the inversion in $({\boldsymbol{\Sigma}}_e(S_n) - \sigma_0 \mathbf{I})^{-1}$ can become singular or ill-conditioned during optimization.
The indicator $\mathbf{1}_{\text{Hessian}}$ activates the mixed-partial regularization only when $\mathbf{H}(\mathbf{r})$ is the Hessian of $\mathbf{G}(\mathbf{r})$, emphasizing the curl-free conditions for hessian of a scalar potential.
In a discretized setting, the Dirac delta function $ \delta(\mathbf{r})$ is approximated as a Kronecker delta, which is non-zero only at the origin, making the residual computationally tractable. We also reiterate that the physics consistency loss $\|\mathscr{L}[\hat{\mathbf{G}}(\mathbf{r})] - \delta(\mathbf{r})\|^2$ only requires a numerical integral in $\mathcal{G}/\mathcal{N}(0)$ where the Green's function is smooth.

\vspace{-0.1in}
\section{Experimental Results}
\label{sec:experiments}
\vspace{-0.1in}
\begin{figure}
  \centering
  \includegraphics[width=\columnwidth]{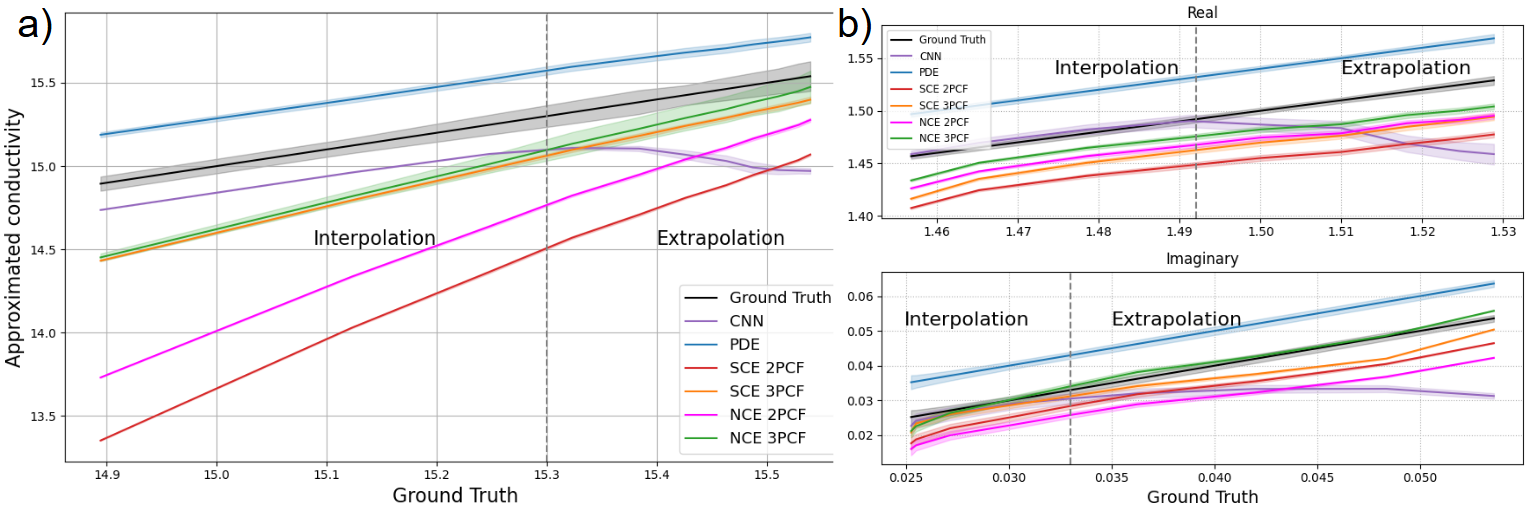}
  \caption{a) Approximate conductivity versus ground truth for CNN, low‐resolution PDE solver, SCE and NCE methods based on 2- and 3PCF. The vertical dashed line marks the transition from interpolation (left) to extrapolation (right). Shaded regions indicate one standard deviation.}
   \label{figure: extrapolation}
   \vspace{-0.1in}
\end{figure}

\paragraph{Data.}
We use 10 distinct parameter settings of bi-phase 2D random fields with contrast-phase volume fraction $\phi = 0.5$,
defined by correlation lengths in the x- and y-directions. 
The correlation length in the x-direction is fixed at 0.1$\%$ of the resolution, while in the y-direction, it varies from 5$\%$ to 60$\%$. For each parameter settings, we generate 100 realizations of 1024 × 1024 microstructure of size and use the same set for both static conduction and wave propagation.
For the static conduction problem, the effective conductivity of the reference phase and contrast phase is set to 5 and 20 respectively; for the wave propagation problem, the dielectric constants are 1 and 2. These property values are chosen arbitrarily to ensure a strong contrast between the two media.
For the wave-propagation problem, we fix the reference-phase wave number $k=10$.
Volume fraction and correlation lengths are chosen so that the ergodicity assumption is satisfied: i.e., the effective conductivities and dielectric constant computed from all 100 microstructure samples using a PDE solver have a small within-group standard deviation (see the gray shade in Fig.~\ref{figure: extrapolation}).
For each microstructure $x$, we collect $(S_2(x),S_3(x),\boldsymbol{\Sigma}_e^*(x))$, where $S_2(x)$ (resp. $S_3(x)$) contains all 2-point (resp. 3-point) correlations averaged over 256 patches (64 × 64 each) randomly sampled from the 1024 × 1024 microstructure. This mimics realistic settings where microstructure reconstruction is only affordable on small material samples yet effective properties can be experimentally measured from a relatively large sample. 

\vspace{-0.1in}
\paragraph{Baselines and the training setting.} 
We generate a high-resolution (1024 × 1024) ground truth benchmark using PDE solvers tailored to specific physics. For static conduction, we employ a conventional steady-state solver, while for wave propagation, we use the Finite-Difference Time-Domain (FDTD) method \citep{FDTD2013}. We then assess three low-resolution (64 × 64) methods against this benchmark: a down-sampled solver serving as a baseline, and two machine learning models. Both machine learning models are trained to predict effective property from spatial correlations. The first is an end-to-end (E2E) convolutional neural network using 2PCF as input. The second is our Neural Contrast Expansion model, which is trained on either 2-point (2PCF) or both 2- and 3-point (3PCF) correlations and learns its kernel corrections via Eq.~\eqref{eq:loss function}.

\begin{figure}
  \centering
  \includegraphics[width=\columnwidth]{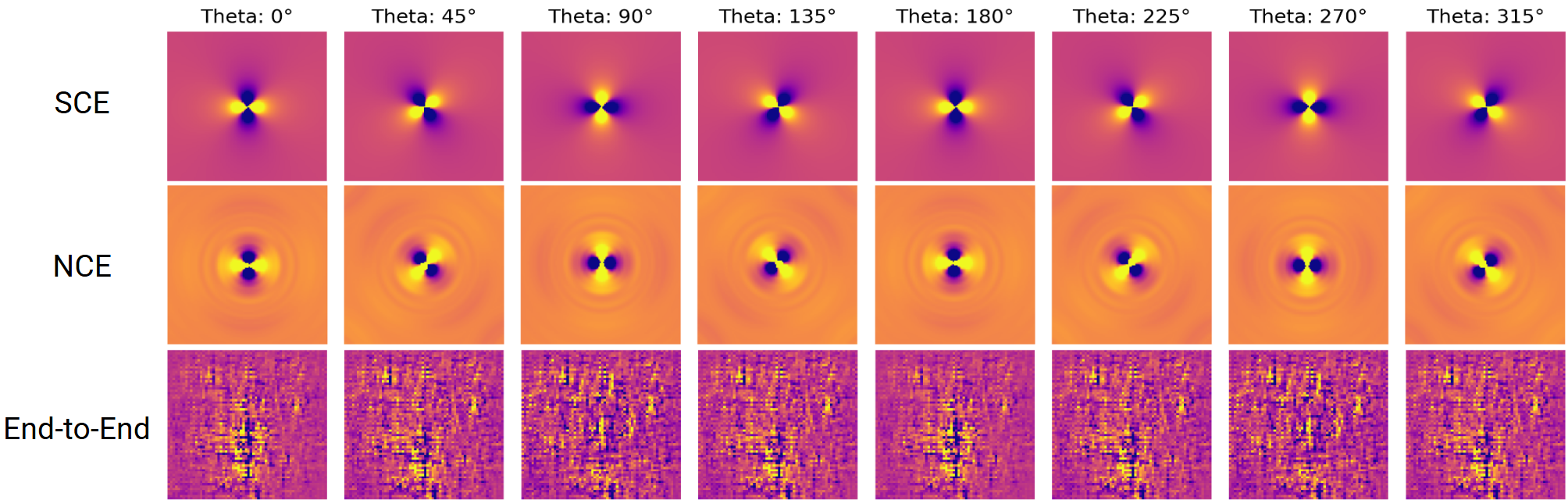}
  \caption{Sensitivity analysis of directional effective conductivity for SCE(top), NCE (middle), and end‐to‐end approach (bottom)—shown across different orientations $\theta$. The color scale (yellow to dark blue) indicates regions of positive versus negative influence.}
   \label{figure: conduction_kernel}
\end{figure}

\begin{figure}
  \centering
  \includegraphics[width=\columnwidth]{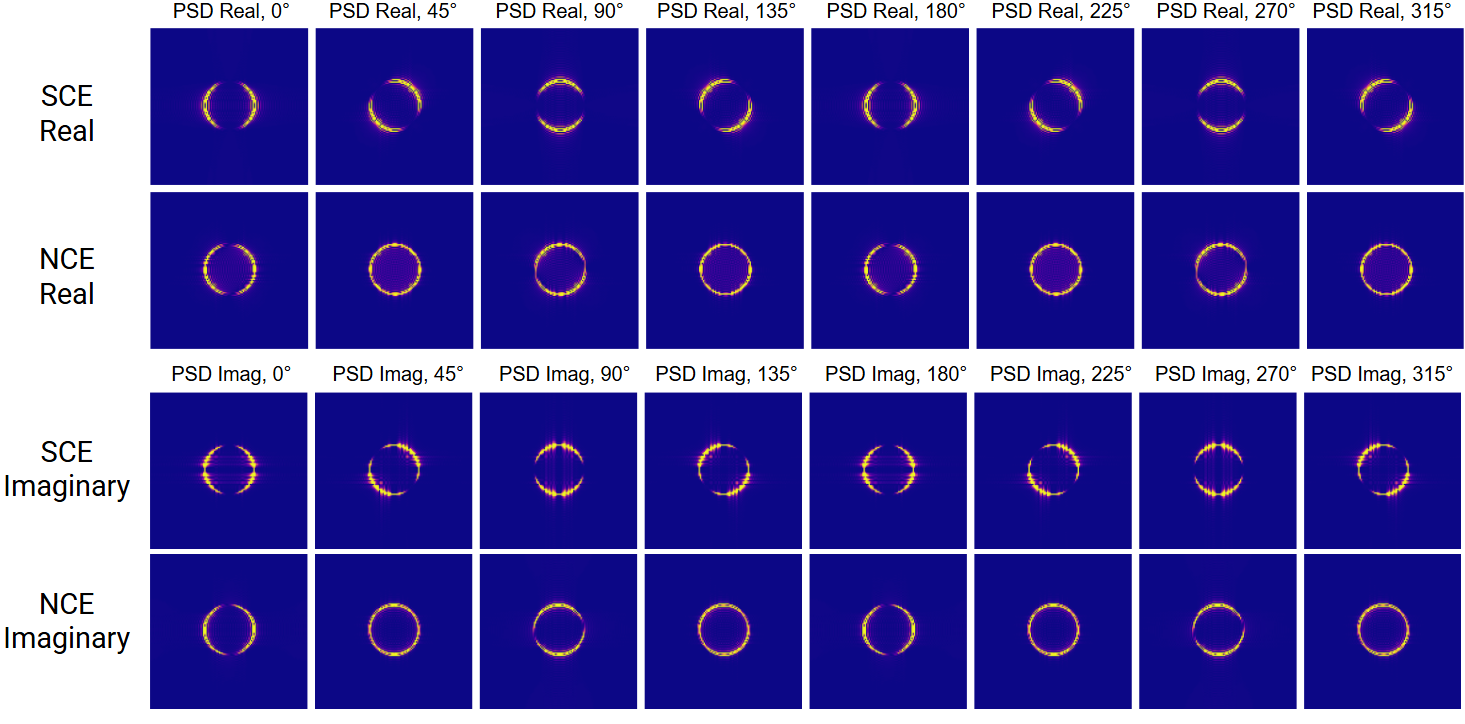}
  \caption{Sensitivity analysis of the real and imaginary parts of the directional effective dielectric constant with respect to power spectral density is shown for both the SCE and NCE approaches across different orientation angles. The color scale (yellow to dark blue) indicates regions of strong positive influence versus negligible influence. (See the time-domain counterpart in Fig. \ref{fig: wave_time_domain})}
   \label{figure: wave_kernel}
   \vspace{-0.1in}
\end{figure}

\vspace{-0.1in}
\paragraph{Interpolation and extrapolation accuracy.} Fig.~\ref{figure: extrapolation} compares the predicted conductivities from all models with the ground truth. Indicated by the vertical dashed line, the test uses random field parameters that are both within and out of the training distribution. 
The standard deviation in the prediction across all 100 realizations are characterized by the shades. 
We have the following observations: 
While both NCE and the E2E model performed well on random fields within the training distribution (interpolation), the E2E model's accuracy deteriorated significantly when presented with novel random fields (extrapolation). In contrast, NCE maintained high accuracy in both regimes for the static conduction and wave propagation problems, demonstrating a robust generalization capability that the purely data-driven approach lacks. The analytical SCE models showed inherent prediction errors, confirming that their truncation error and infinite-domain assumption is a significant limitation for finite-sized samples, a gap that NCE's learned kernel mitigates.

\vspace{-0.1in}
\paragraph{Explainability for static conduction.}
This superior generalization is a direct consequence of NCE's ability to learn physically meaningful relationships, which we validated through sensitivity analysis. Specifically, we compute the partial derivative of directional electrical conductivity with respect to corresponding 2PCF. The directional physical property is measured along an angle $\theta$ from the x-axis. 
For static conduction, Fig.~\ref{figure: conduction_kernel} shows that NCE correctly learns the characteristic quadrupole sensitivity pattern provided by the analytical SCE model. This pattern physically signifies that directional conductivity is enhanced by phase connectivity parallel to the applied field and impeded by connectivity in the orthogonal direction. The localized nature of the influence indicates that short-range spatial correlations dominate charge transport mechanisms, providing a clear design implications on synthesizing composite materials.
On the other hand, the sensitivity map of the E2E model (Fig.~\ref{figure: conduction_kernel}, bottom) is not physically interpretable since the model does not utilize physics-informed inductive biases.
Thus while the E2E model can achieve reasonable in-distribution prediction accuracy (Fig.~\ref{figure: extrapolation}), it lacks physical explainability to provide material design guidelines.

\vspace{-0.1in}
\paragraph{Design insights from wave propagation.}
We analyze sensitivity in the Fourier domain by differentiating the directional component of the effective dielectric tensor with respect to the Fourier transform of the two-point correlation (equivalently the power spectral density). Both SCE and NCE produce the same physical picture: sensitivity concentrates on the $k$-shell $|\mathbf{k}|=k_0$, forming a ring in Fig.~\ref{figure: wave_kernel}. Crucially, the imaginary part $\operatorname{Im}(\epsilon_e)$ controls scattering loss; the maps show that loss is driven by spectral power placed on this ring (and, directionally, on its angular sectors). Therefore, to reduce scattering loss at a target frequency $k_0$, design the microstructure to suppress spectral power on $|\mathbf{k}|=k_0$ (isotropically or in selected sectors). This directly serves low-loss applications such as transparent electromagnetic windows \citep{yang2022emwindow} and on-chip photonic waveguides \citep{Tran2018photonics}, where minimizing $\operatorname{Im}(\epsilon_e)$ is essential.

This design insight connects directly to the concept of hyperuniformity: a novel class of exotic disordered patterns possessing hidden long-range order \cite{To03, To18a}. Hyperuniform microstructures strongly suppress long-wavelength scattering and thus possess an exclusion zone in their spectral density \cite{To18a}. By shaping the exclusion to cover or intersect the $k_0$ ring, one suppresses the long-wavelength scattering channels that feed $\operatorname{Im}(\epsilon_e)$, achieving low loss while retaining control of $\operatorname{Re}(\epsilon_e)$ via modes away from the ring. Indeed, the resulting microstructures shown in Fig. \ref{figure: 1}(a) are verified to be hyperuniform and consistent with those constructed based on specific spectral densities with anisotropic and isotropic exclusion zones \cite{shi2023computational}. Because NCE learns a kernel consistent with the governing physics, its sensitivity maps become actionable design levers: weighting a candidate power spectrum by the learned sensitivity identifies the loss-dominant bands to remove and the angular sectors in which to remove them. In this way, Fig.~\ref{figure: wave_kernel} functions not only as diagnosis but as a compact, quantitative recipe for engineering low-loss, anisotropic dielectrics at prescribed $(\omega,k_0)$.

\vspace{-0.1in}
\section{Conclusion}
\vspace{-0.1in}
In this work, we propose Neural Contrast Expansion (NCE), a data-driven model that preserves the analytical architecture of strong-contrast expansion while learning a surrogate PDE kernel directly from structure–property data. Across static conduction and electromagnetic wave propagation, NCE matched or exceeded the predictive accuracy of purely data-driven baselines while retaining physically interpretable sensitivities. In conduction, NCE recovered the quadrupolar sensitivity pattern connecting directional effective conductivity to phase connectivity; in waves, Fourier-domain sensitivities concentrated on the $|\mathbf{k}| = k_0$ shell, cleanly separating how spectral content controls $\operatorname{Re}(\epsilon_e)$ versus $\operatorname{Im}(\epsilon_e)$. This yields actionable design rules: suppress power on the $k_0$-ring to reduce scattering loss in $\operatorname{Im}(\epsilon_e)$ and adjust off-ring content to tailor anisotropy in $\operatorname{Re}(\epsilon_e)$—linking directly to (stealthy) hyperuniform microstructures that enforce spectral exclusion zones. Practically, NCE requires only macroscopic property measurements, not field-level supervision, and its physics-informed kernel parameterization mitigates SCE’s finite-domain and truncation errors while enabling principled extrapolation to novel microstructures.

Limitations suggest clear next steps. Our analysis focuses on randomness entering zeroth- or second-order terms of linear, self-adjoint PDEs; extending to skew-adjoint and weakly nonlinear settings will require a careful re-examination of the SCE derivation and its solvability assumptions. A complementary direction is to jointly optimize over selected higher-order correlations which could boost sample and computational efficiency. Ultimately, NCE provides a practical, physics-grounded framework that translates microstructural statistics into actionable guidelines, accelerating more reliable material discovery and design.

\newpage
\bibliographystyle{unsrt}
\bibliography{ref}
\newpage
\appendix
\section{Appendix}

\subsection{Derivation of SCE for 2D conductivity and discussion about its computational complexity}
\label{sec:sce}
We now provide the derivation of the effective linear conductivity under the assumption of an infinite boundary (i.e., a static electric field at infinity)~\citep{torquato2002random}. For clarity, we focus on a continuous two-dimensional domain $\mathcal{G}$. Dependence on the random field realization $x$ is omitted when possible. The governing equation is
\begin{equation}
\nabla\cdot\bigl(\sigma\nabla\phi\bigr)=0,
\end{equation}
with the Neumann boundary condition $\nabla\phi(x_\mathbf{i}) = -\mathbf{E}_0$ and with $\sigma(x_\mathbf{i})=\sigma_0$ for all $\mathbf{i} \in \partial \mathcal{G}$. Writing the local conductivity as
\[
\sigma(x_\mathbf{i}) = \sigma_0 + (\sigma_1-\sigma_0)x_\mathbf{i} = \sigma_0 + \sigma'\,x_\mathbf{i},
\]
the PDE becomes:
\begin{equation}
\nabla^2\phi = -\frac{1}{\sigma_0}\nabla\cdot\bigl(\sigma'\nabla\phi\bigr).
\label{eq:sce_laplacian}
\end{equation}

This can be solved by introducing an appropriate Green’s function for the Laplace operator. When
boundary is at infinity and in 2D, we have:
\begin{equation}
G(\mathbf{i},\mathbf{i}') = -\frac{1}{2\pi\sigma_0}\ln\|\mathbf{i}-\mathbf{i}'\|_2,
\label{eq:inf_greens_function}
\end{equation}
which satisfies
\[
\nabla\cdot\nabla G(\mathbf{i},\mathbf{i}') = -\frac{1}{\sigma_0}\delta(\mathbf{i}-\mathbf{i}'),
\]
where $\delta(\cdot)$ is the 2D Dirac delta function. The solution to Eq. \eqref{eq:sce_laplacian} in terms of the Green’s function is shown to be:
\begin{equation}
\begin{aligned}
\phi(\mathbf{i}) &= \phi_0(\mathbf{i}) - \int_{\mathcal{G}} G(\mathbf{i},\mathbf{i}')\,\nabla'\cdot\mathbf{P}(\mathbf{i}')\,d\mathbf{i}'\\[1mm]
&= \phi_0(\mathbf{i}) + \int_{\mathcal{G}} \nabla'G(\mathbf{i},\mathbf{i}')^\top\,\mathbf{P}(\mathbf{i}')\,d\mathbf{i}',\\
& \text{(integral by part and $X = 0$ at $\partial \mathcal{G}$)}\\
& = \phi_0(\mathbf{i}) - \int_{\mathcal{G}} \nabla G(\mathbf{i},\mathbf{i}')^\top \mathbf{P}(\mathbf{i}') d\mathbf{i}', \\
& (\nabla' G(\mathbf{i},\mathbf{i}') = -\nabla G(\mathbf{i},\mathbf{i}')),
\end{aligned}
\label{eq:sce1}
\end{equation}
where $\phi_0(\mathbf{i})=-\langle\mathbf{E}_0,\mathbf{i}\rangle$ is the potential for the homogeneous reference phase, $\mathbf{P}=-\sigma'\nabla\phi$ is the polarization field, and $\nabla'\cdot$ is the divergence with respect to $\mathbf{i}'$.

Since $\nabla\phi(\mathbf{i})=-\mathbf{E}(\mathbf{i})$ and $\mathbf{P}$ is proportional to $\mathbf{E}$, we take the gradient of both sides of Eq.~\eqref{eq:sce1} to create a fixed-point iteration with respect to $\mathbf{P}$. However, it should be first noted that
\[
\nabla G(\mathbf{i},\mathbf{i}') = -\frac{1}{2\pi\sigma_0}\frac{\mathbf{r}}{r^2},\quad\text{with}\quad \mathbf{r}=\mathbf{i}-\mathbf{i}',
\]
is singular at $\mathbf{r}=\mathbf{0}$. To handle this, we decompose the integration domain by separating a small disk 
\[
\mathcal{B}_\epsilon = \{\mathbf{i}'\in\mathcal{G}\mid \|\mathbf{i}-\mathbf{i}'\|_2\le \epsilon\}
\]
from the rest of the domain $\mathcal{G}/\epsilon$. Applying the divergence theorem shows that the contribution from $\mathcal{B}_\epsilon$ yields
\begin{equation}
\nabla\int_{\partial\mathcal{B}_\epsilon} \langle G(\mathbf{i},\mathbf{i}')\,\mathbf{P}(\mathbf{i}'),\mathbf{n}'\rangle ds' = -\frac{1}{2\sigma_0}\mathbf{P}(\mathbf{i}),
\label{eq:sce_epsilon}
\end{equation}
where $\mathbf{n}'$ is the outward normal on $\partial\mathcal{B}_\epsilon$. For $d > 2$, the gradient in Eq.~\eqref{eq:sce_epsilon} is in general $-\frac{1}{d\sigma_0}\mathbf{P}$. Therefore, differentiating Eq.~\eqref{eq:sce1} gives
\begin{equation}
\mathbf{E}(\mathbf{i}) = \mathbf{E}_0 - \frac{1}{d\sigma_0}\mathbf{P}(\mathbf{i}) + \int_{\mathcal{G}/\epsilon}\nabla^2 G(\mathbf{i},\mathbf{i}')\,\mathbf{P}(\mathbf{i}')\,d\mathbf{i}',
\label{eq:sce_linear}
\end{equation}
where for 2D we set $d=2$. 

Defining 
\[
\mathbf{F} = \Bigl(1+\frac{\sigma'}{d\sigma_0}\Bigr)\mathbf{E} \quad\text{and}\quad l = \frac{\sigma'}{1+\frac{\sigma'}{d\sigma_0}},
\]
so that $\mathbf{P} = l\,\mathbf{F}$, and letting $\mathbf{H} = \nabla^2 G$, Eq.~\eqref{eq:sce_linear} can be rewritten as
\begin{equation}
\mathbf{F} = \mathbf{E}_0 + \mathbf{H}\,\mathbf{P}.
\label{eq:sce_linear2}
\end{equation}
Iteratively applying this relation yields
\begin{equation}
\begin{aligned}
\mathbf{P} &= l\,\mathbf{E}_0 + l\,\mathbf{H}\,l\,\mathbf{E}_0 + l\,\mathbf{H}\,l\,\mathbf{H}\,l\,\mathbf{E}_0 + \cdots\\[1mm]
&= \Bigl(l + l\,\mathbf{H}\,l + l\,\mathbf{H}\,l\,\mathbf{H}\,l + \cdots\Bigr)\mathbf{E}_0\\[1mm]
&= \mathbf{S}\,\mathbf{E}_0.
\end{aligned}
\label{eq:sce_iterate_linear}
\end{equation}
Ensemble averaging Eq.~\eqref{eq:sce_iterate_linear} along with Eq.~\eqref{eq:sce_linear2} leads to
\[
\langle\mathbf{F}\rangle = \Bigl(\langle\mathbf{S}^{-1}\rangle+\mathbf{H}\Bigr)\langle\mathbf{P}\rangle.
\]
A further expansion (involving ensemble averages of the products of $l$ and $\mathbf{H}$) yields an expression for $\langle\mathbf{F}\rangle$ in terms of the microstructural correlation functions. Introducing the expansion parameter
\[
\beta = \frac{\sigma_1-\sigma_0}{\sigma_1+(d-1)\sigma_0},
\]
one finds that
\[
\langle l(\mathbf{i})\rangle = d\sigma_0\beta\,\phi,\quad \langle l(\mathbf{i})\,l(\mathbf{i}')\rangle = (d\sigma_0\beta)^2 S_2(\mathbf{i},\mathbf{i}'),
\]
for all points in $\mathcal{G}$ (with $\phi$ denoting the phase density). With further algebra, the operator
\[
\mathbf{Q} := \langle\mathbf{S}^{-1}\rangle+\mathbf{H}
\]
can be recast as
\[
\mathbf{Q} = \frac{1}{d\sigma_0\beta^2\phi^2}\Bigl(\beta\phi\,\mathbf{I} - \sum_{n=2}^{\infty} \mathbf{A}_n\,\beta^n\Bigr),
\]
where $\mathbf{A}_n$ is defined in Eq.~\eqref{eq:A}.

Relating the polarization field to the effective conductivity via
\[
\langle\mathbf{P}\rangle = \mathbf{L}_e\langle\mathbf{F}\rangle
\]
with
\[
\mathbf{L}_e = d\sigma_0\bigl(\boldsymbol{\Sigma}_e-\sigma_0\mathbf{I}\bigr)\Bigl[(d-1)\sigma_0\mathbf{I}+\boldsymbol{\Sigma}_e\Bigr]^{-1},
\]
one finally arrives at the SCE series for the effective conductivity in Eq.~\eqref{eq_SCE}.
Note that Eq.~\eqref{eq:sce_iterate_linear} is an expansion of the fixed-point iteration for solving the linear system of equation defined in Eq.~\eqref{eq:sce_linear} with respect to $\mathbf{E}$. By focusing on the polarization field $\mathbf{P}$, we introduce an expansion parameter $\beta$ that has value between 0 and 1 when $\sigma_1 > \sigma_0$. This derivation underpins the SCE method, naturally leading to the infinite series representation for the effective conductivity, as given in Eq.~\eqref{eq_SCE}.

Critically, Eq. \eqref{eq_SCE} shows that  effective linear properties such as conductivity can be analytically explained by two decoupled factors: the Green's function that is only related to the property-dependent physics, and the NPCFs that are only related to the random field. 
Lem.~\ref{lemma:npcf_error} characterizes the computational complexity of SCE for approximating effective linear properties:
\begin{lemma}\label{lemma:npcf_error}
    Let $\boldsymbol{\Sigma}_e(X)$ be a ground-truth effective linear property for $X$ and $\hat{\boldsymbol{\Sigma}}_e(X)$ be its approximation up to an order $N$. Let the approximation error be defined by the Frobenius norm: $\epsilon(X) := \|\boldsymbol{\Sigma}_e(X) - \hat{\boldsymbol{\Sigma}}_e(X)\|_2$. Given $\epsilon>0$, $\epsilon(X) \leq \epsilon$ can be achieved with a computational cost of $\mathcal{O}(\epsilon^{-|\beta| \ln(G)})$. 
\end{lemma}
\begin{proof}
    By reorganizing \eqref{eq_SCE}, we have:
    \begin{equation}
        \boldsymbol{\Sigma}_e = (\mathbf{A}+\mathbf{D})^{-1}(\mathbf{B}+\mathbf{D}),
    \end{equation}
    where $\mathbf{A}:= (\beta \phi - \beta^2\phi^2)\mathbf{I} - \Delta_N$, $\mathbf{B}:= \beta^2\phi^2 (d-1) + \beta \phi \mathbf{I} - \Delta_N$, $\mathbf{D}:= \sum_{n=N+1}^G \mathbf{A}_n \beta^n$, and $\Delta_N := \sum_{n=2}^N \mathbf{A}_n \beta^n$. The finite-order SCE approximation is
    \begin{equation}
        \hat{\boldsymbol{\Sigma}}_e = \mathbf{A}^{-1}\mathbf{B},
    \end{equation}
    and the approximation error is 
    \begin{equation}
        \epsilon^{prop} := \|\hat{\boldsymbol{\Sigma}}_e - \boldsymbol{\Sigma}_e\|_2.
    \end{equation}
    We now derive an upper bound on $\epsilon^{prop}$. First
    \begin{equation}
        \begin{aligned}
            \epsilon^{prop} & = \| (\mathbf{A}+\mathbf{D})^{-1}(\mathbf{B}+\mathbf{D}) - \mathbf{A}^{-1}\mathbf{B} \|_2 \\
            & = \| ((\mathbf{A}+\mathbf{D})^{-1}-\mathbf{A}^{-1})\mathbf{B} + (\mathbf{A}+\mathbf{D})^{-1}\mathbf{D}\|_2 \\
            & \leq \| \mathbf{A}^{-1}(\mathbf{A}+\mathbf{D}-\mathbf{A})(\mathbf{A}+\mathbf{D})^{-1}\mathbf{B} \|_2 \\
            &+ \|(\mathbf{A}+\mathbf{D})^{-1}\mathbf{D}\|_2 \\
            & \leq \|(\mathbf{A}+\mathbf{D})^{-1}\|_2 \|\mathbf{D}\|_2 (\|\mathbf{A}^{-1}\|_2\|\mathbf{B}\|_2 + 1). 
        \end{aligned}
    \end{equation}

    Since $\mathbf{D}$ is a bounded sum of high-order residuals, $\|\mathbf{D}\|_2 \leq \|\mathbf{A}\|_2$ for large enough $N$. Under this condition, we have
\begin{equation}
\begin{aligned}
(\mathbf{A}+\mathbf{D})^{-1}
&= \bigl[\mathbf{A}\bigl(\mathbf{I} + \mathbf{A}^{-1}\mathbf{D}\bigr)\bigr]^{-1} \\
&= \bigl(\mathbf{I} + \mathbf{A}^{-1}\mathbf{D}\bigr)^{-1}\,\mathbf{A}^{-1} \\
&= \sum_{k=0}^{\infty} \bigl(-\mathbf{A}^{-1}\mathbf{D}\bigr)^k\,\mathbf{A}^{-1}.
\end{aligned}
\end{equation}
    Therefore
    \begin{equation}
    \begin{aligned}
        \epsilon^{prop} \leq \frac{\|\mathbf{A}^{-1}\|_2}{1-\|\mathbf{A}^{-1}\|_2 \|\mathbf{D}\|_2} \|\mathbf{D}\|_2 (\|\mathbf{A}^{-1}\|_2\|\mathbf{B}\|_2 + 1).
    \end{aligned}
    \end{equation}
    $\epsilon^{prop}$ monotonically increases with $\|\mathbf{D}\|_2$, i.e., to achieve $\epsilon^{prop} \leq \epsilon$, we need $\|\mathbf{D}\|_2 \leq C_1 \epsilon$ with some $C_1 > 0$.
    
    Inspecting the structure of $\mathbf{D}$ to have
    \begin{equation}
        \mathbf{D} = \sum_{n = N+1}^{G} \beta^n \delta_n,
    \end{equation}
    where the matrix $\delta_n$ absorbs the convolution between the field kernel and the NPCF at order $n$. Let $C_2$ be the largest element of $\delta_n$ for any $n \in \{N+1,...,G\}$ and notice that $|\beta| < 1$, we can bound $\|\mathbf{D}\|_2$ by
    \begin{equation}
        \|\mathbf{D}\|_2 \leq C_2 \sum_{n = N+1}^{G} |\beta|^n = C_3 |\beta|^N. 
    \end{equation}

    Therefore, to achieve $\|\mathbf{D}\|_2 \leq C_1 \epsilon$, we need $N \geq \frac{\ln( C_1\epsilon/C_3)}{\ln(|\beta|)}$. Since the computational complexity of the SCE approximation is determined by that of the highest order NPCF, this complexity is $\mathcal{O}\left(G^{\frac{\ln( C_1\epsilon/C_3)}{\ln(|\beta|)}-1}\right) = \mathcal{O}\left(\epsilon^{-|\beta| \ln(G)}\right)$.
\end{proof}

\paragraph{Remarks.} For regular grid $\mathcal{G}$, we have $\ln(G) = d \ln(s^{-1})$. Therefore the worst-case complexity of the SCE approximation is $\mathcal{O}(\epsilon^{|\beta|d\ln(s)})$. While the exact computational complexities of the PDE and SCE approaches are not directly comparable due to their involvement of problem specific paramters, the following observations are useful:
(1) Both approaches are affected by the phase-wise property gap $\sigma_1 - \sigma_0$. For PDE, the gap affects the condition number of the resultant system of equations, which in turn affects the convergence rate of the solver. For SCE, the gap directly controls $\beta$. In the trivial case of $\sigma_1 = \sigma_0$, SCE directly provides the analytical solution, making PDE unnecessary. (2) Real-world random fields often have limited correlation lengths, i.e., there exists $r_0 > 0$ such that for any $\mathbf{i}_1, \mathbf{i}_2 \in \mathcal{G}$ and $\|\mathbf{i}_1 -\mathbf{i}_2\|_2 \geq r_0$, $\mathbb{E}[X_{\mathbf{i}_1} X_{\mathbf{i}_2}] = \mathbb{E}[X_{\mathbf{i}_1}]\mathbb{E}[X_{\mathbf{i}_2}]$. With this property, Lem.~\ref{lemma:exponential} shows that the proportion of non-zero $\Delta_N$ among $\mathcal{C}_N$ reduces exponentially along $N$. 
Therefore, we hypothesize that it is possible to replace $\{\mathcal{C}_n\}_{n=1}^N$ with a significantly smaller subset $\mathcal{C}^*$ to achieve both good approximation of $\boldsymbol{\Sigma}_e$ and a much lower computational complexity.

\begin{lemma}\label{lemma:exponential}
Given correlation order $N$ and the grid $\mathcal{G}$, let $\gamma(N)$ be the fraction of $N$-point configurations $c \in \mathcal{C}_N$ for which the total correlation $\Delta_N(c) > 0$, and let $\Gamma(\cdot)$ be the the Gamma function. If $r_0 < \left(\Gamma(d/2+1)/\pi^{d/2}\right)^{1/d}$, $\gamma(N)$ decreases exponentially with $N$.
\end{lemma}
\begin{proof}

The total number of possible $N$-point configurations (considered as ordered tuples) is $G^{N-1}$ considering that the first point is fixed to $\mathbf{0}$.
An $N$-point configuration $c$ is said to be \textit{connected} if every pair of points in $c$ is connected via a path of points within $c$ such that consecutive points in the path are within distance $r_0$ of each other.
The number of connected configurations of size $N$ can be approximated as $\mu^{N-1}$, where $\mu$ is the average number of ways to add a new site to a connected cluster, also known as the \textit{branching factor}. For a regular grid, $\mu$ can be approximated as $V(r_0)/l^d$, where $V(r_0)$ is the volume of a sphere with radius $r_0$ in $\mathbb{R}^d$, and $l^d$ is the volume occupied by each grid point.

We now show that if $c$ is not connected, i.e., there exists clusters $c_1,...c_K$ so that $c = \bigcup_{k \in [K]}c_k$ and $c_k \bigcap c_{k'} = \empty$ for all $k \neq k'$, then $\Delta(c) = 0$. To show this, first notice that $\Delta(c)$ is a joint cumulant and is invariant to the permutation of points in $c$. We can now consider $c = c_0 \bigcup c_1$ and $c_0 \bigcap c_1 = \empty$, so that for any $X_0 \in c_0$ and $X_1 \in c_1$, $\mathbb{E}[X_0X_1] = \mathbb{E}[X_0] \mathbb{E}[X_1]$. 
We can reorder the elements of $\Delta(c)$ so that for some $k \in [N]$, points involved in the rows above the $k$th row all belong to $c_0$, and the new elements introduced on and after the $k$th row belong to $c_1$. Thus the elements on the $k-1$ and $k$th rows are linearly dependent and therefore the determinant is 0. 

Let the side length of the grid be $1$. The fraction $\gamma(N)$ is therefore
\begin{equation}
\begin{aligned}
\gamma(N) &= 
\frac{\text{Number of connected configurations}}{\text{Total configurations}} \\
&\approx \frac{\mu^{N-1}}{G^{N-1}} 
= \left( \frac{\pi^{d/2}r_0^d}{\Gamma(d/2+1)} \right)^{N-1}.
\end{aligned}
\end{equation}
Therefore, when $r_0 < \left(\Gamma(d/2+1)/\pi^{d/2}\right)^{1/d}$, $\gamma(N)$ decreases exponentially with $N$. The threshold for $r_0$ is 0.61 for $d=2$ and 0.64 for $d=3$.
\end{proof}

\subsection{SCE breaks down for first--order randomness}
\label{sec: first_order_rand}
Here we investigate the case where the randomness appears \emph{only in the first--order term} (e.g., a drift/advection field) of a linear second-order PDE.\footnote{The derivation in this subsection is assisted by GPT5 Pro and verified by the authors.} We derive the corresponding Lippmann--Schwinger representation, identify the associated kernels and their singular structure, and then show precisely why SCE cannot be carried out: the perturbation is non--self--adjoint, there is no local ``polarizability'' linking the polarization to the driving field, and interface distributions enter in a way that breaks the usual correlation--function framework. A narrow exception occurs when the drift is a potential field, which can be removed by a similarity transform that converts the problem into one with zeroth--order randomness, where SCE applies.


Fix $d\ge 2$. Consider the linear PDE
\begin{equation}\label{eq:PDE}
\mathcal{L}u := -\diver(\bA_0 \grad u)\;+\;\bB(\bx)\!\cdot\!\grad u\;+\;c_0\,u \;=\; s,
\end{equation}
where $\bA_0\in\R^{d\times d}$ is constant, symmetric positive definite (SPD), $c_0\ge 0$ is constant, and the \emph{randomness} is only in the drift $\bB(\bx)$. We take the baseline
\begin{equation}\label{eq:L0}
\mathcal{L}_0 u := -\diver(\bA_0 \grad u) + c_0 u, 
\qquad \mathcal{L}_0 G_0(\br)=\delta(\br).
\end{equation}
We will impose a uniform macroscopic gradient $\bm{E}_0$ in the baseline medium so that the baseline solution is $u_0(\bx)=-\bm{E}_0\!\cdot\!\bx$ (or a periodic surrogate), and denote the total gradient (``field'') by
\[
\bm{E}(\bx):=\grad u(\bx), \qquad \langle \bm{E}\rangle=\bm{E}_0.
\]


\paragraph{Self--adjointness.}
The formal $L^2$ adjoint (integration by parts) of the drift term is
\[
(\bB\!\cdot\!\grad)^{\!*} = -\bB\!\cdot\!\grad - (\diver \bB)\,,
\]
hence the adjoint of $\mathcal{L}$ is
\begin{equation}\label{eq:Ladjoint}
\mathcal{L}^{\!*} = -\diver(\bA_0 \grad \cdot)\;-\;\bB\!\cdot\!\grad\;-\;(\diver\bB)\;+\;c_0.
\end{equation}
Therefore $\mathcal{L}$ is self--adjoint in $L^2$ \emph{iff} $\bB\equiv \bm{0}$. Even if $\diver \bB=0$ (divergence--free drift), the operator remains non--self--adjoint because of the $-\bB\!\cdot\!\grad$ in \eqref{eq:Ladjoint}.


\paragraph{Lippmann--Schwinger representation for first--order randomness.}
Subtract $\mathcal{L}_0 u_0=s$ from $\mathcal{L}u=s$ to get
\[
\mathcal{L}_0(u-u_0) = -\bB\!\cdot\!\grad u.
\]
Convolving with $G_0$ and using $\mathcal{L}_0 G_0 = \delta$ gives the \emph{resolvent identity}
\begin{equation}\label{eq:LS0}
u = u_0 - G_0 \conv (\bB\!\cdot\!\grad u).
\end{equation}
Use the product rule to expose only \emph{multiplications} (no derivatives on $u$) inside the convolution:
\[
\bB\!\cdot\!\grad u = \diver(\bB u) - (\diver \bB)\,u.
\]
Integrate by parts inside the convolution:
\[
\int_{\R^d} G_0(\bx-\by)\,\diver_{\by}(\bB(\by) u(\by))\,d\by
= - \int_{\R^d} \grad G_0(\bx-\by)\!\cdot\!(\bB(\by) u(\by))\,d\by.
\]
Thus
\begin{equation}\label{eq:LSu}
\boxed{\;
u(\bx) = u_0(\bx) 
+ \int_{\R^d} \grad G_0(\bx-\by)\!\cdot\!\underbrace{\big[\bB(\by)\,u(\by)\big]}_{:=\,\bm{P}(\by)}\,d\by
- \int_{\R^d} G_0(\bx-\by)\,\underbrace{\big[(\diver\bB)(\by)\,u(\by)\big]}_{:=\,q(\by)}\,d\by.
\;}
\end{equation}
Differentiating \eqref{eq:LSu} yields the gradient equation
\begin{equation}\label{eq:LSE}
\boxed{\;
\bm{E}(\bx) = \bm{E}_0
+ \int_{\R^d} \underbrace{\grad\grad G_0(\bx-\by)}_{:=\,\mathbf{K}_0(\bx-\by)} : \bm{P}(\by)\,d\by
- \int_{\R^d} \grad G_0(\bx-\by)\, q(\by)\,d\by.
\;}
\end{equation}

\paragraph{Singular structure of the kernels.}
The Hessian kernel $\mathbf{K}_0(\br)=\grad\grad G_0(\br)$ has a short--range $\delta$--singular part:
\begin{equation}\label{eq:delta-split}
\boxed{\;\mathbf{K}_0(\br) = -\boldsymbol{\Lambda}_2\,\delta(\br)\;+\;\mathbf{T}^{\,L}(\br),\qquad 
\boldsymbol{\Lambda}_2 := \tfrac{1}{d}\,\bA_0^{-1}. \;}
\end{equation}
The remainder $\mathbf{T}^{\,L}$ is integrable and satisfies $\int_{\R^d}\mathbf{T}^{\,L}(\br)\,d\br=\boldsymbol{\Lambda}_2$. In contrast, the vector kernel $\grad G_0(\br)$ is \emph{odd} and has
\begin{equation}\label{eq:grad-no-delta}
\int_{\R^d} \grad G_0(\br)\,d\br=\bm{0}
\quad\text{and}\quad
\text{no }\delta\text{ part.}
\end{equation}
(An easy Fourier--domain check: $\widehat{\grad G_0}(\bk)= i\bk / (\bk\!\cdot\!\bA_0\bk + c_0)$ is odd in $\bk$, so it cannot produce a constant term, i.e., a $\delta$ in real space.)

Plugging \eqref{eq:delta-split} into \eqref{eq:LSE} gives
\begin{equation}\label{eq:E-renorm-candidate}
\boxed{\;
\bm{E} = \bm{E}_0 \;-\; \boldsymbol{\Lambda}_2\,\bm{P} \;+\; \mathbf{T}^{\,L}\conv \bm{P} \;-\; (\grad G_0)\conv q,
\qquad \bm{P} := \bB\,u,\quad q:=(\diver\bB)\,u.
\;}
\end{equation}

\paragraph{How SCE works in the self--adjoint cases.}
For comparison, when the randomness is in the \emph{second--order} term, such as in the case of conduction, write $\bA(\bx)=\bA_0+X(\bx)$ and define
\[
\bm{P}=X\,\bm{E}.
\]
Equation \eqref{eq:E-renorm-candidate} then becomes
\(
\bm{E} = \bm{E}_0 - \boldsymbol{\Lambda}_2 X \bm{E} + \mathbf{T}^{\,L}\conv (X\bm{E}),
\)
so one can \emph{move the local term} to the left and \emph{invert it pointwise}:
\[
(\Id + \boldsymbol{\Lambda}_2 X)\,\bm{E} = \bm{E}_0 + \mathbf{T}^{\,L}\conv (X\bm{E})
\quad\Longrightarrow\quad
\bm{P} = M\,\bm{E}_0 + M\,\mathbf{T}^{\,L}\conv \bm{P},
\]
with the \emph{local polarizability}
\[
M := X\,(\Id+\boldsymbol{\Lambda}_2 X)^{-1}.
\]
This produces the SCE series for the effective tensor in terms of $N$-point correlation functions of the microstructure. A similar (scalar) story holds when the randomness is purely in the zeroth--order term.

\paragraph{Why SCE breaks for first--order randomness.}
Return to \eqref{eq:E-renorm-candidate}, where
\[
\bm{P}=\bB\,u \quad\text{and}\quad q=(\diver\bB)\,u.
\]
The barriers to SCE are 
\begin{itemize}
    \item No local relation $\bm{P}=M\,\bm{E}$:
In the self--adjoint cases, $\bm{P}$ depends on $\bm{E}$ by \emph{multiplication}. Here, $\bm{P}$ depends on the \emph{potential} $u$, not on its gradient $\bm{E}$. There is no local matrix $M(\bx)$ so that $\bB(\bx)\,u(\bx)=M(\bx)\,\bm{E}(\bx)$ for all admissible $u$.

\begin{lemma}[No local polarizability]
Suppose there exists a (measurable) matrix field $M(\bx)$ with $\bm{P}(\bx)=M(\bx)\,\bm{E}(\bx)$ for all solutions $u$. Then $\bB\equiv \bm{0}$.
\end{lemma}
\begin{proof}
Fix $\bx_0$ and choose a test solution $u$ that is locally constant near $\bx_0$ (e.g., a smooth bump function flattened near $\bx_0$). Then $\bm{E}(\bx_0)=\bm{0}$ but $u(\bx_0)$ can be nonzero, hence $\bm{P}(\bx_0)=\bB(\bx_0)\,u(\bx_0)$. The assumed relation gives $\bm{P}(\bx_0)=M(\bx_0)\bm{E}(\bx_0)=\bm{0}$, forcing $\bB(\bx_0)=\bm{0}$. Since $\bx_0$ is arbitrary, $\bB\equiv\bm{0}$.
\end{proof}

\item The $q$--term has no $\delta$ to renormalize:
The last term in \eqref{eq:E-renorm-candidate} involves the kernel $\grad G_0$, which has no $\delta$ part \eqref{eq:grad-no-delta}. Thus, unlike the $\mathbf{K}_0$ term, there is no local singularity to ``absorb'' into a pointwise transform. Any attempt to eliminate $q$ would be nonlocal (via $\mathcal{L}_0^{-1}$), defeating the SCE philosophy.

\item Interface distributions pollute the microstructure coefficients:
In many microstructure models, properties are piecewise constant by phase. If $\bB(\bx)=\sum_{p=1}^M \mathcal{I}_p(\bx)\,\bm{b}_p$ with phase indicators $\mathcal{I}_p\in\{0,1\}$ and constant vectors $\bm{b}_p$, then
\[
\diver \bB = \sum_{p=1}^M \bm{b}_p\!\cdot\!\grad \mathcal{I}_p.
\]
Distributionally, $\grad \mathcal{I}_p$ is a \emph{surface measure} concentrated on the phase boundaries:
\[
\grad \mathcal{I}_p = -\,\bm{n}\,\delta_{\partial\Omega_p},
\]
where $\bm{n}$ is the unit normal and $\delta_{\partial\Omega_p}$ is a Dirac mass on the interface. Consequently $q=(\diver\bB)\,u$ is \emph{supported on surfaces}, and ensemble averages of the form $\langle (\grad G_0)\conv q\rangle$ depend on \emph{surface--surface} statistics (geometry of interfaces), not on the usual \emph{volume} $N$-point correlation functions $\{S_n\}$ that power SCE. This is a structural mismatch.
\item No symmetric energy (variational) structure:
Self--adjoint randomness leads to a positive energy form
\(
\int (\grad u)\!\cdot\!\bA(\bx)\,\grad u + \int c(\bx)\,u^2,
\)
from which bounds and renormalized series follow. The drift contribution satisfies
\[
\int \phi\,(\bB\!\cdot\!\grad u) 
= - \int u\,\big(\bB\!\cdot\!\grad \phi + (\diver\bB)\,\phi\big),
\]
which is not sign--definite and provides no coercive variational principle. This removes the key positivity used in SCE resummations.
\end{itemize}


\paragraph{A narrow exception: potential drift can be removed.}
There is one structurally important exception:
\begin{theorem}[Self--adjointization by similarity transform]\label{thm:selfadj}
If $\bB(\bx)=\bA_0\,\grad \psi(\bx)$ for some scalar $\psi$, and $\bA_0$ is constant SPD, then the change of variables $u = e^{\psi/2} v$ converts $\mathcal{L}$ into a self--adjoint operator with a \emph{random zeroth--order potential}:
\begin{equation}\label{eq:similarity}
\mathcal{L}u = e^{\psi/2}\left[-\diver(\bA_0 \grad v) 
+ \Big(c_0 \;+\; \tfrac{1}{4}\,\grad\psi\!\cdot\!\bA_0\grad\psi 
\;-\; \tfrac{1}{2}\,\diver(\bA_0 \grad\psi)\Big)\,v \right].
\end{equation}
\end{theorem}
\begin{proof}[Proof sketch]
Compute $\grad u = e^{\psi/2}\big(\grad v + \tfrac{1}{2} v\,\grad\psi\big)$ and use the product rule with $\bA_0$ constant:
\begin{align*}
\diver(\bA_0 \grad u)
&= \diver\!\left(e^{\psi/2} \bA_0\big(\grad v + \tfrac{1}{2} v\,\grad\psi\big)\right) \\
&= e^{\psi/2}\Big[\diver(\bA_0\grad v) + (\grad v)\!\cdot\!\bA_0\grad\psi 
+ \tfrac{1}{2} v\,\diver(\bA_0\grad\psi) + \tfrac{1}{4} v\,\grad\psi\!\cdot\!\bA_0\grad\psi \Big].
\end{align*}
Also $\bB\!\cdot\!\grad u = e^{\psi/2}\big((\bA_0\grad\psi)\!\cdot\!\grad v + \tfrac{1}{2} v\,\grad\psi\!\cdot\!\bA_0\grad\psi\big)$. Substituting in $\mathcal{L}u=-\diver(\bA_0\grad u)+\bB\!\cdot\!\grad u + c_0 u$ cancels the $(\grad v)\!\cdot\!\bA_0\grad\psi$ terms and yields \eqref{eq:similarity}.
\end{proof}

\begin{remark}
In one spatial dimension ($d=1$), every scalar drift is a gradient; therefore the transform always exists and reduces the problem to one with zeroth--order randomness, where SCE applies. In $d\ge 2$, a generic stationary random drift is \emph{not} a gradient (\,$\curl \bB\neq \bm{0}$\,), so the transform is unavailable.
\end{remark}
\begin{figure}
  \centering
  \includegraphics[width=\columnwidth]{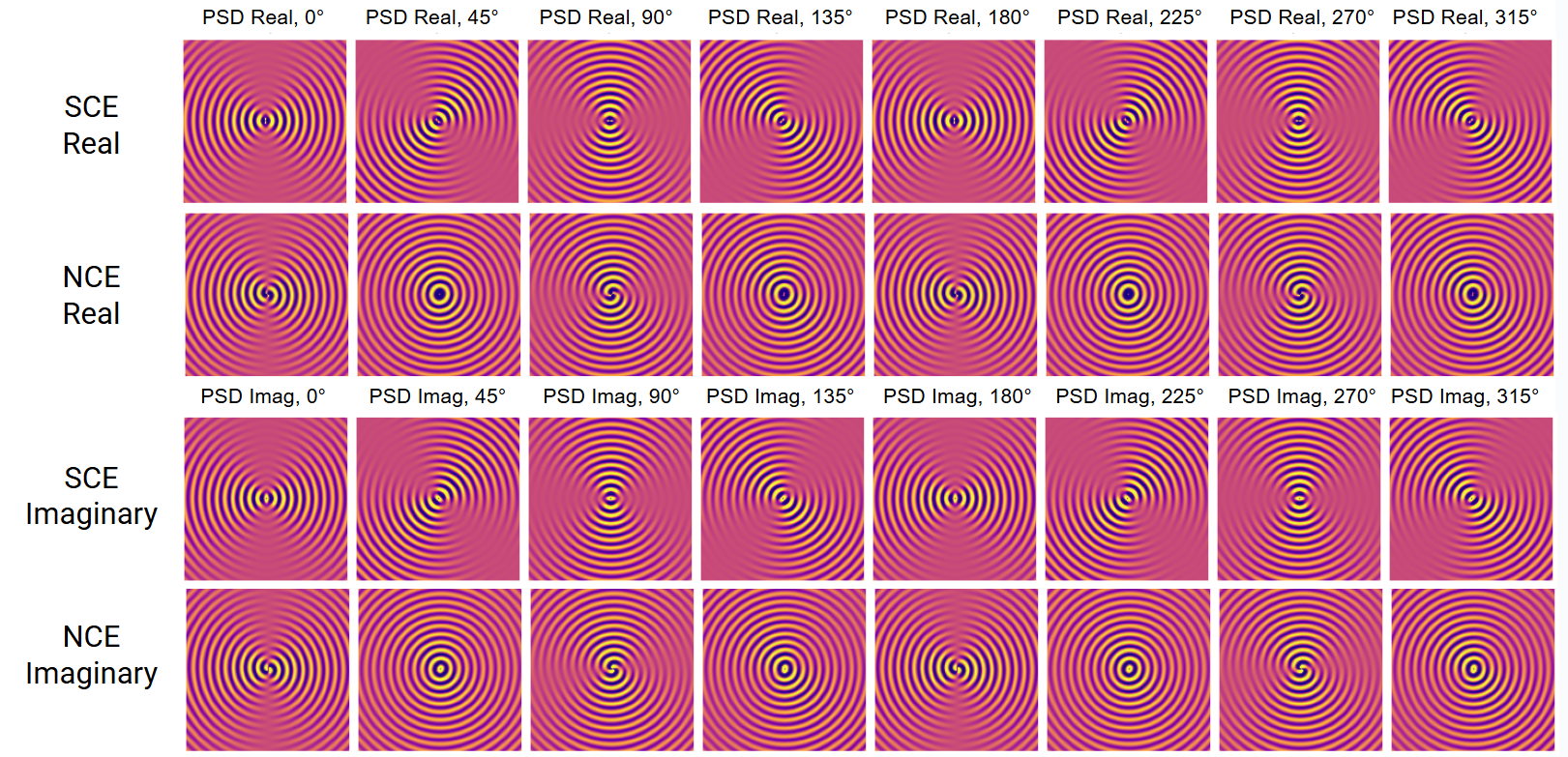}
  \caption{Sensitivity analysis of the real and imaginary parts of the directional effective dielectric constant with respect to 2PCF is shown for both the SCE and NCE approaches across different orientation angles. The color scale (yellow to dark blue) indicates regions of strong positive influence versus negligible influence.}
  \label{fig: wave_time_domain}
\end{figure}

\subsection{NCE–LLM Integration: Extracting Physical Insights}
\label{sec:LLM}
In this section, we present a blinded study with the GPT-5 Thinking model: memory and prior-conversation referencing were disabled, the model was shown only the sensitivity maps (without disclosing the underlying physical property), and it was tasked with proposing designs that meet a specified engineering objective under a fixed volume fraction. Screenshots of the conversation are provided below.

\begin{figure}
  \centering
  \includegraphics[width=\columnwidth]{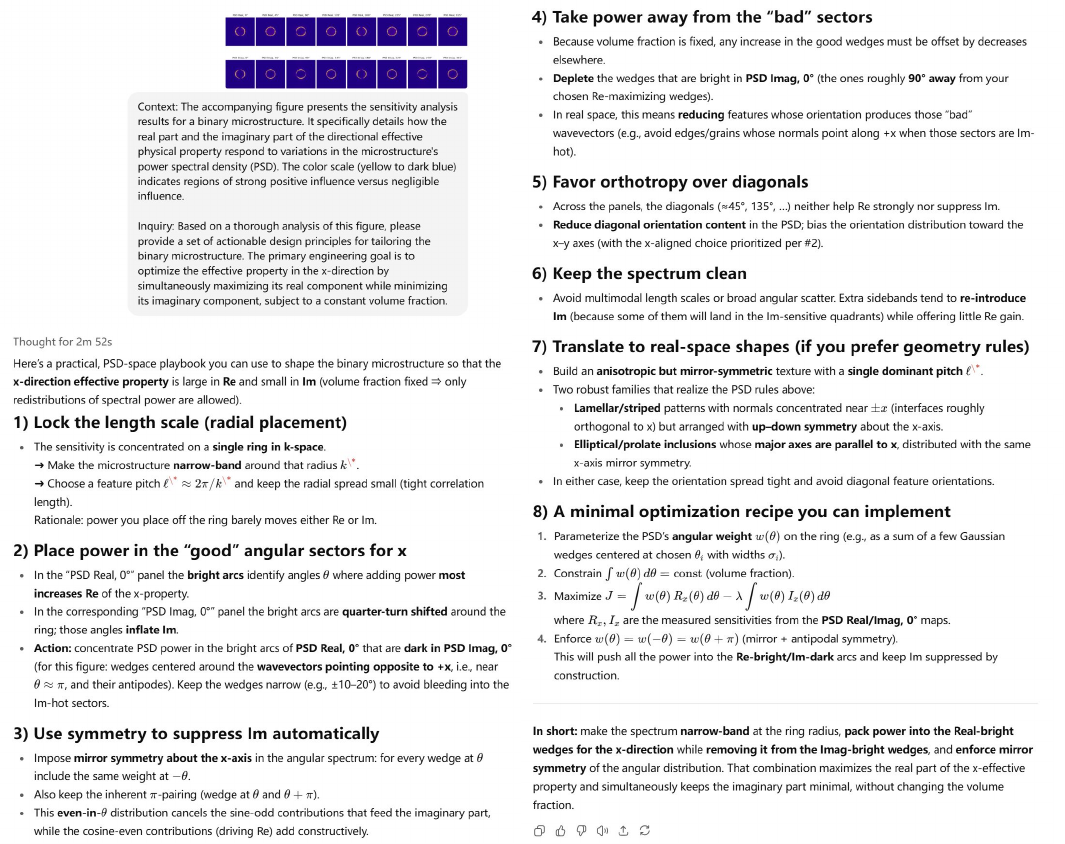}
  \caption{The primary engineering goal is to optimize the effective property in the x-direction by simultaneously maximizing its real component while minimizing its imaginary component, subject to a constant volume fraction.}
\end{figure}

\begin{figure}
  \centering
  \includegraphics[width=\columnwidth]{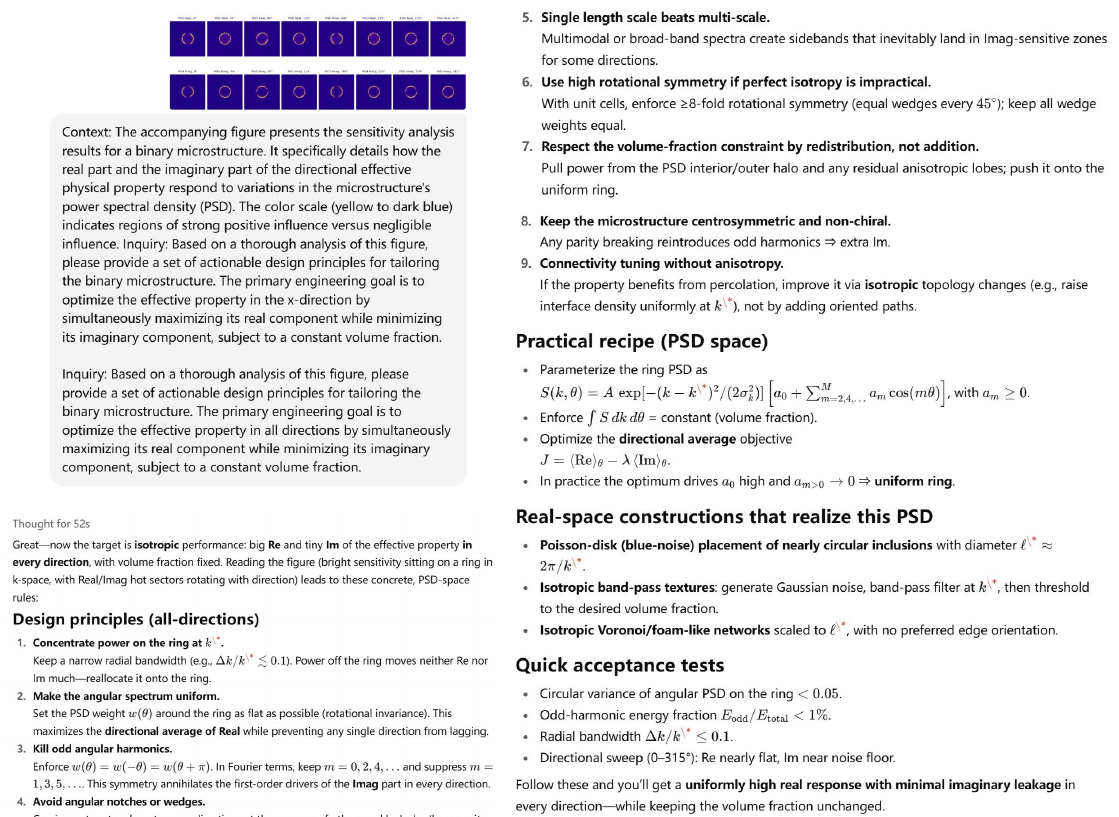}
  \caption{The primary engineering goal is to optimize the effective property in all directions by simultaneously maximizing its real component while minimizing its imaginary component, subject to a constant volume fraction.}
\end{figure}

\begin{figure}
  \centering
  \includegraphics[width=\columnwidth]{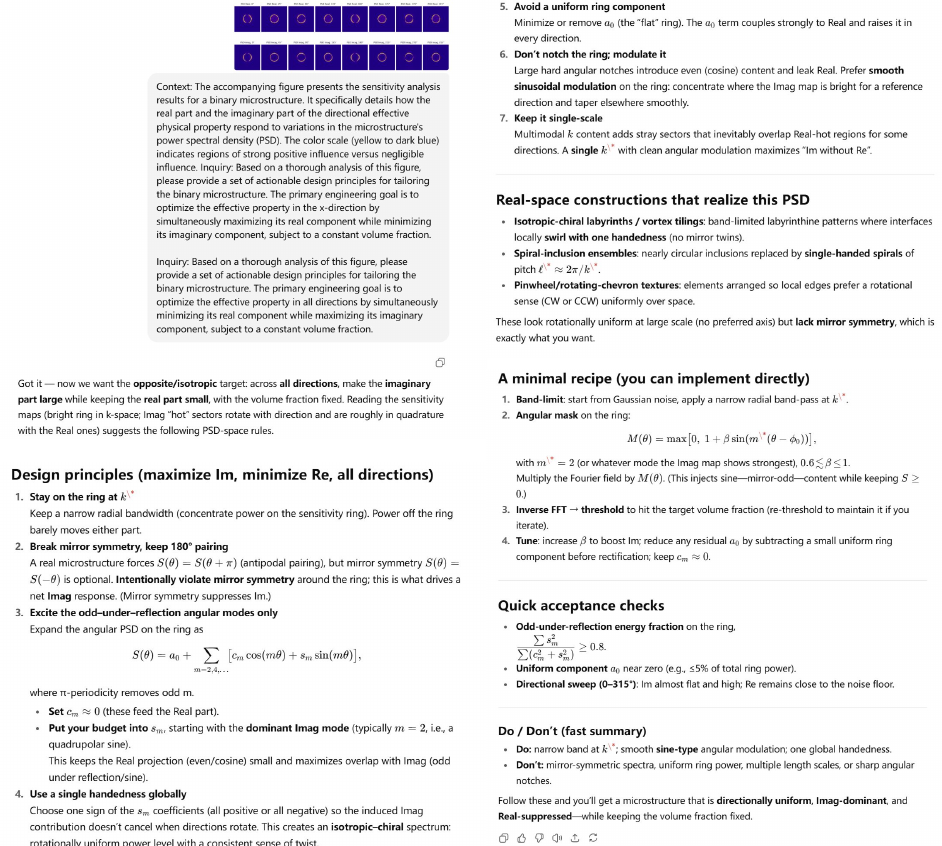}
  \caption{The primary engineering goal is to optimize the effective property in all directions by simultaneously minimizing its real component while maximizing its imaginary component, subject to a constant volume fraction.}
\end{figure}

\end{document}